\documentclass{article}
\usepackage{ijcai20}

\usepackage{times}

\usepackage{soul}
\usepackage{url}
\usepackage[hidelinks]{hyperref}
\usepackage[utf8]{inputenc}
\usepackage[small]{caption}
\usepackage{graphicx}
\usepackage{amsmath}
\usepackage{booktabs}
\usepackage[utf8]{inputenc}
\usepackage[english]{babel}
\usepackage{amssymb, amsthm, mathtools,amsmath}
\usepackage[margin=1in]{geometry}
\usepackage{color, colortbl}
\usepackage{bm}
\usepackage{xspace}
\usepackage{array, multirow}
\usepackage{booktabs}
\usepackage{microtype}
\usepackage{subfigure}
\usepackage[normalem]{ulem}
\usepackage{physics}
\usepackage{calc}
\usepackage{tabularx}
\usepackage{pgfplotstable}
\usepackage{numprint}
\usepackage{mathdots}
\usepackage{graphicx,caption}
\usepackage{wrapfig}
\usepackage{algorithm}
\usepackage{algorithmic}
\usepackage{pifont}
\usepackage{pgfplots}
\usetikzlibrary{matrix}
\usepgfplotslibrary{groupplots}
\usepgfplotslibrary{statistics}
\pgfplotsset{compat=newest}
\usetikzlibrary{fillbetween}
\usepackage{xcolor}
\usepackage{amsfonts}
\usepackage{dsfont}
\usepackage[shortlabels]{enumitem}

\DeclareMathOperator*{\argmax}{argmax}
\DeclareMathOperator*{\argmin}{argmin}
\DeclareMathOperator*{\expect}{ \mathbb{E}}

\usepackage{times}
\usepackage{hyperref}
\usepackage{custom}

\newcolumntype{C}[1]{>{\centering\arraybackslash}p{#1}}
\newcolumntype{L}[1]{>{\arraybackslash}p{#1}}

\newcommand{\Omit}[1]{}

\newtheorem{theorem}{Theorem}
\newtheorem*{theorem*}{Theorem}

\newtheorem*{remark*}{Remark}

\title{Robust Neural Networks using Randomized Adversarial Training}

\author{
Alexandre Araujo$^{1,2}$\footnote{Contact Author}\and
Laurent Meunier$^{1,3}$\and
Rafael Pinot$^{1,4}$\And
Benjamin Negrevergne$^{1}$\\
\affiliations
$^1$PSL, Université Paris-Dauphine, Miles Team\\
$^2$Wavestone\\
$^3$Facebook AI Research\\
$^4$CEA, Université Paris-Saclay\\
\emails
\{firstname.lastname\}@dauphine.psl.eu  
}

\newcommand{\ltwo}{\ensuremath{\ell_2}\xspace}
\newcommand{\linf}{\ensuremath{\ell_\infty}\xspace}
\newcommand{\E}{\mathbb{E}}

\DeclareMathOperator*{\Vol}{Vol}
\newcommand{\prob}{\mathbb{P}}

\begin{document}

\maketitle

\begin{abstract}
    This paper tackles the problem of defending a neural network against adversarial attacks crafted with different norms (in particular $\ell_\infty$ and $\ell_2$ bounded adversarial examples). It has been observed that defense mechanisms designed to protect against one type of attacks often offer poor performance against the other. We show that $\ell_\infty$ defense mechanisms cannot offer good protection against $\ell_2$ attacks and vice-versa, and we provide both theoretical and empirical insights on this phenomenon. Then, we discuss various ways of combining existing defense mechanisms in order to train neural networks robust against both types of attacks. Our experiments show that these new defense mechanisms offer better protection when attacked with both norms.
\end{abstract}

\section{Introduction}
\label{intro}
Deep neural networks achieve state of the art performances in a variety of domains such as natural language processing~\cite{radford2018Language}, image recognition~\cite{He_2016_CVPR} and speech
recognition~\cite{hinton2012deep}. However, it has been shown that such neural networks are vulnerable to {\em adversarial examples}, i.e. imperceptible variations of natural examples, crafted to deliberately mislead the models~\cite{globerson2006nightmare,biggio2013evasion,Szegedy2013IntriguingPO}. Since their discovery, a variety of algorithms have been developed to generate adversarial examples (a.k.a. attacks), for example FGSM \citep{goodfellow2014explaining}, PGD \citep{madry2018towards} and C\&W \citep{carlini2017towards}, to mention the most popular ones.

Because it is difficult to characterize the space of visually imperceptible variations of a natural image, existing adversarial attacks use surrogates that can differ from one attack to another. For example, \citet{goodfellow2014explaining} use the \linf norm to measure the distance between the original image and the adversarial image whereas \citet{carlini2017towards} use the \ltwo norm.  When the input dimension is low, the choice of the norm is of little importance because the \linf and \ltwo balls overlap by a large margin, and the adversarial examples lie in the same space. An important insight in this paper is to observe that the overlap between the two balls  diminishes exponentially quickly as the dimensionality of the input increases. For typical image datasets with large dimensionality, the two balls are mostly disjoint. As a consequence, the \linf-bounded and the \ltwo-bounded adversarial examples lie in different area of the space, and it explains why \linf defense mechanisms perform poorly against \ltwo attacks and vice-versa. 

We show that this insight is crucial to design defense mechanisms that are robust against both types of attacks, and we advocate for the design of models that incorporate defense mechanisms against both \linf and \ltwo attacks. Then we evaluate strategies (existing and new ones) to mix up existing defense mechanisms. In particular, we evaluate the following strategies:
  \begin{enumerate}[(a)]
        \item {\em Mixed Adversarial Training} (MAT), a training procedure inspired by  {\em Adversarial Training}~\cite{goodfellow2014explaining}. It is based on augmenting training batches using \emph{both} \linf and \ltwo adversarial examples. This method defends well against both norms for PGD attacks, but fails against C\&W attacks.
        \item {\em Mixed noise injection} (MNI), a technique that consists in noise injection at test time~\cite{KolterRandomizedSmoothing, pinot2019theoretical}. We evaluate different noises and their mixture. This method defends better against C\&W attacks, but does not obtain good results against PGD attacks for \linf norm.
        \item {\em Randomized Adversarial Training} (RAT), a solution to benefit from the advantages of both \linf adversarial training, and \ltwo randomized defense. As we will show, RAT offers the best trade-off between defending against PGD and C\&W attacks.
    \end{enumerate}

The rest of this paper is organized as follows. In Section~\ref{sec:preliminaries}, we recall the principle of existing attacks and defense mechanisms. In Section~\ref{sec:no_free_lunch}, we conduct a theoretical analysis to show why  the \linf defense mechanisms cannot be robust against \ltwo attacks and vice-versa. We then corroborate this analysis with empirical results using real adversarial attacks and defense mechanisms. In Section~\ref{sec:building_defense_mechanisms}, we discuss various strategies to mix defense mechanisms, conduct comparative experiments, and discuss the performance of each strategy.

\section{Preliminaries on Adversarial Attacks and Defense Mechanisms}
\label{sec:preliminaries}

Let us first consider a standard classification task with an input space $\mathcal{X}=[0,1]^d$ of dimension $d$,  an output space $\mathcal{Y}=[K]$ and a data distribution $\mathcal D$ over $\mathcal X \times \mathcal Y$. We assume the model $f_\theta$ has been trained to minimize  a loss function $\mathcal{L}$ as follows:
 \begin{equation}
    \min_{\theta} \E_{(x,y) \sim \mathcal{D}} \left[ \mathcal{L}(f_\theta(x), y) \right]. 
    \label{eqn:classification}
 \end{equation}

\noindent
In this paper, we consider $N$-layers neural network models, therefore the model is a composition of $N$ non-linear parametric functions $\phi_{\theta_i}$ (i.e. $f_\theta= \phi^{(N)}_{\theta_N}\circ \cdots \circ \phi^{(1)}_{\theta_1}$).
 
\subsection{Adversarial attacks}
\label{subsec:adversarial_attacks}

Given an input-output pair $(x,y) \sim \mathcal{D}$, an {\em adversarial attack} is a procedure that produces a small perturbation $\tau \in  \mathcal X$  such that $f_\theta(x + \tau) \neq y$. To discover the damaging perturbation $\tau$ of $x$, existing attacks can adopt one of the two following strategies:  (i)  maximizing the loss $\mathcal L(f_\theta(x + \tau), y)$ under some constraint on $\norm{\tau}_p$, with $p \in \{0, \cdots, \infty\}$ (a.k.a. loss maximization); or (ii)  minimizing $\norm{\tau}_p$ under some constraint on the loss $\mathcal L(f_\theta(x + \tau), y)$ (a.k.a. perturbation minimization). 

\textbf{(i) Loss maximization.} In this scenario, the procedure maximizes the loss objective function, under the constraint that the $\ell_p$ norm of the perturbation remains bounded by some value $\epsilon$, as follows:  

\begin{equation}
  \argmax_{\norm{\tau}_p \leq \epsilon} \mathcal{L}(f_\theta(x+\tau),y).
  \label{eqn::lossmax}
\end{equation}

The typical value of $\epsilon$ depends on the value $p$ of the  norm $\norm{\cdot}_p$ considered in the problem setting. In order to compare \linf and \ltwo attacks of similar strength, we choose values of $\epsilon_\infty$ and $\epsilon_2$ (for \linf and \ltwo norms respectively) which result in \linf and \ltwo balls of equivalent volumes. For the particular case of CIFAR-10, this would lead us to choose $\epsilon_\infty = 0.03$ and $\epsilon_2 = 0.8$ which correspond to the maximum values chosen empirically to avoid the generation of visually detectable perturbations. 
The current state-of-the-art method to solve Problem~(\ref{eqn::lossmax}) is based on a projected gradient descent (PGD)~\cite{madry2018towards} of radius~$\epsilon$. Given a budget $\epsilon$, it recursively computes
\begin{equation}
    x^{t+1}=\prod_{B_p(x,\epsilon)}\left(x^t
    +\alpha \argmax_{\delta\text{ s.t. }||\delta||_p\leq1} \left(\Delta^t|\delta \right)\right)
    \label{eqn::projectionPGD}
\end{equation}
where $B_p(x,\epsilon) = \{ x+\tau \text{~s.t.~} \norm{\tau}_p \leq \epsilon\}$, $\Delta^t=\nabla_x\mathcal{L}\left(f_\theta\left(x^t\right),y\right)$, $\alpha$ is a gradient step size, and $\prod_S$ is the projection operator on $S$. Both PGD attacks with $p=2$, and $p=\infty$ are currently used in the literature as state-of-the-art attacks for the loss maximization problem.

\textbf{(ii) Perturbation minimization.}  This type of procedures search for the perturbation that has the minimal $\ell_p$ norm, under the constraint that $\mathcal{L}(f_\theta(x+\tau),y)$ is bigger than a given bound $c$:
  
  \begin{equation}
      \argmin_{\mathcal{L}(f_\theta(x+\tau),y) \geq c} 
      \norm{\tau}_p.
      \label{eqn::normmin}
  \end{equation}
  The value of $c$ is typically chosen depending on the loss function $\mathcal{L}$. For example, if $\mathcal{L}$ is the $0/1$ loss, any $c>0$ is acceptable.
  Problem~(\ref{eqn::normmin}) has been tackled by~\citet{carlini2017towards}, leading to the strongest method known so far. (Denoted C\&W attack in the rest of the paper.) It aims at solving the following Lagrangian relaxation of Problem~(\ref{eqn::normmin}):
  \begin{equation}
    \argmin_{\tau} \norm{\tau}_p+ \lambda \times g(x+\tau)
    \label{eqn::CWproblem}
\end{equation}
where $g(x+\tau)<0$ if and only if $\mathcal{L}(f_\theta(x+\tau),y) \geq c$. 
The authors use a change of variable $\tau=\tanh(w)-x$ to ensure that $-1 \leq x+\tau \leq 1$, a binary search to optimize the constant $c$, and Adam or SGD to compute an approximated solution. The C\&W attack is well defined both for $p=2$, and $p=\infty$, but there is a clear empirical gap of efficiency in favor of the \ltwo attack. Accordingly, for this work, we only consider C\&W as an \ltwo attack solving a norm minimization problem.

\subsection{Defense mechanisms}
\label{subsec:defense_mechanisms}

\paragraph{Adversarial Training.}
\label{paragraph:adversarial_training}

{\em Adversarial Training} (AT)  was introduced by \citet{goodfellow2014explaining} and later improved by \citet{madry2018towards} as a first defense mechanism to train robust neural networks. It consists in augmenting training batches with adversarial examples generated during the training procedure. At each training step, the standard training procedure from Equation~\ref{eqn:classification} is replaced with a $\min$ $\max$ objective function to minimize the expected value of maximum (perturbed) loss, as follows:

\begin{equation}
    \min_{\theta}\expect_{(x, y)\sim \mathcal{D}} \left[ \max_{\norm{\tau}_p \leq \epsilon} \mathcal{L} \left( f_{\theta}(x+\tau), y \right) \right].
\end{equation}
\noindent
In the case where $p=\infty$, this technique offers good robustness  against $\ell_\infty$ attacks \cite{athalye2018obfuscated}.  AT can also be performed using other kinds of attacks (including strong \ltwo attacks such as C\&W albeit at a much higher computational cost). 
However, as we will discuss in Section~\ref{sec:no_free_lunch}, \linf adversarial training  offers poor protection against \ltwo adversarial attacks and vice-versa.

\paragraph{Noise injection mechanisms.}\label{subsec:randomized_training}

Another important technique to design robust models  against adversarial attacks is to inject noise in the model. Injecting a noise vector $\eta$ at inference time results in a randomized neural network $\tilde{f}_\theta:= f_\theta(x + \eta).$ 

In contrast with Adversarial Training,  noise injection mechanisms are, in certain cases, provably robust against adversarial examples as discussed by \citet{pinot2019theoretical,KolterRandomizedSmoothing}. Empirical results have also demonstrated their efficiency against \ltwo adversarial attacks~\cite{DBLP:journals/corr/abs-1811-09310}. These works focus however on Gaussian and Laplace distributions a.k.a generalized Gaussian of order $2$, and $1$ respectively. As the limit of a generalized Gaussian density~\cite{gengauss} when $p\rightarrow\infty$ is a Uniform distribution, we also investigate the injection of uniform noise to defend against $\linf$ attacks.

\begin{figure*}[ht]
  \centering
  \begin{minipage}{.33\linewidth}
    \centering
    \includegraphics[scale=0.15]{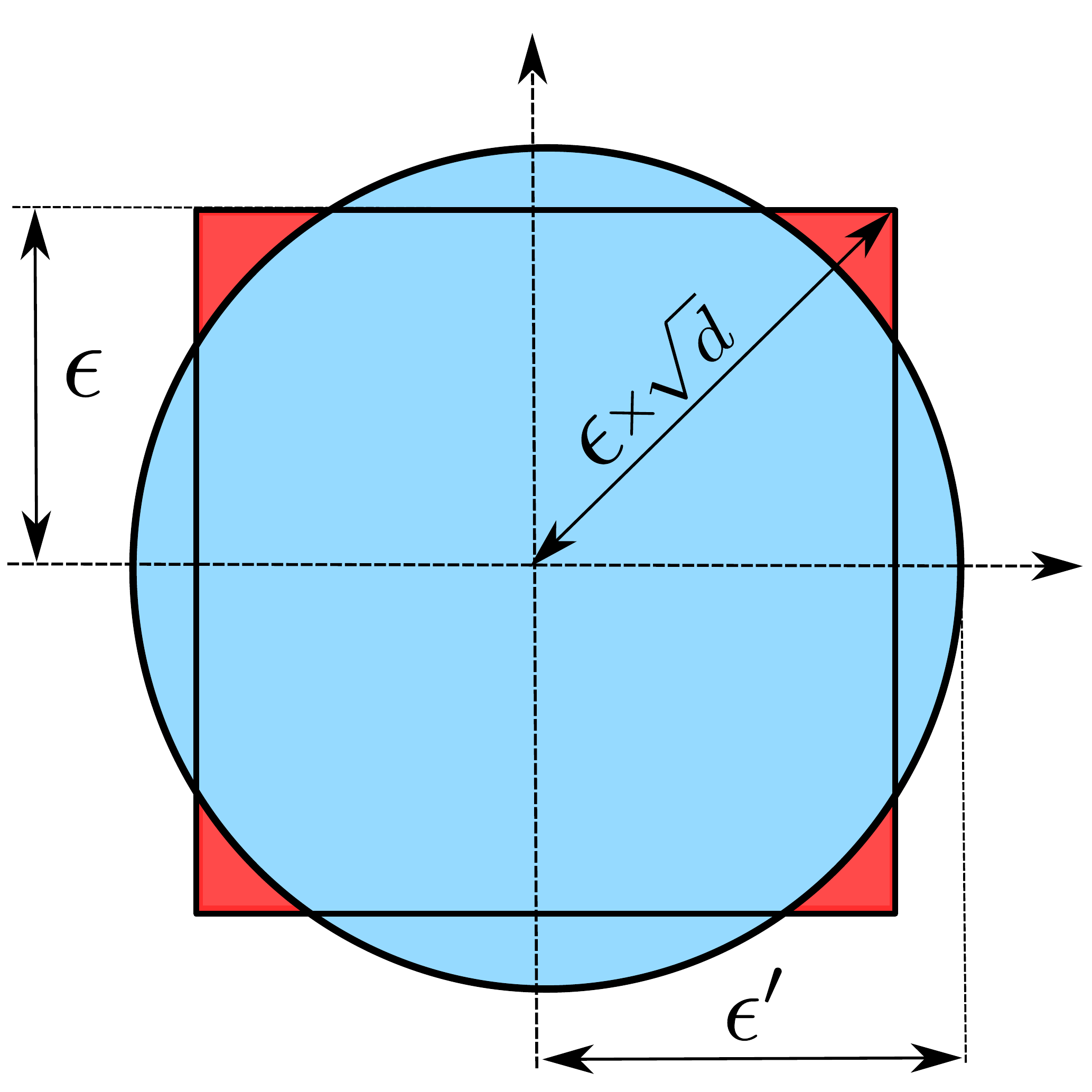}\\(a)
  \end{minipage}
  \begin{minipage}{.33\linewidth}
    \centering
    \includegraphics[scale=0.15]{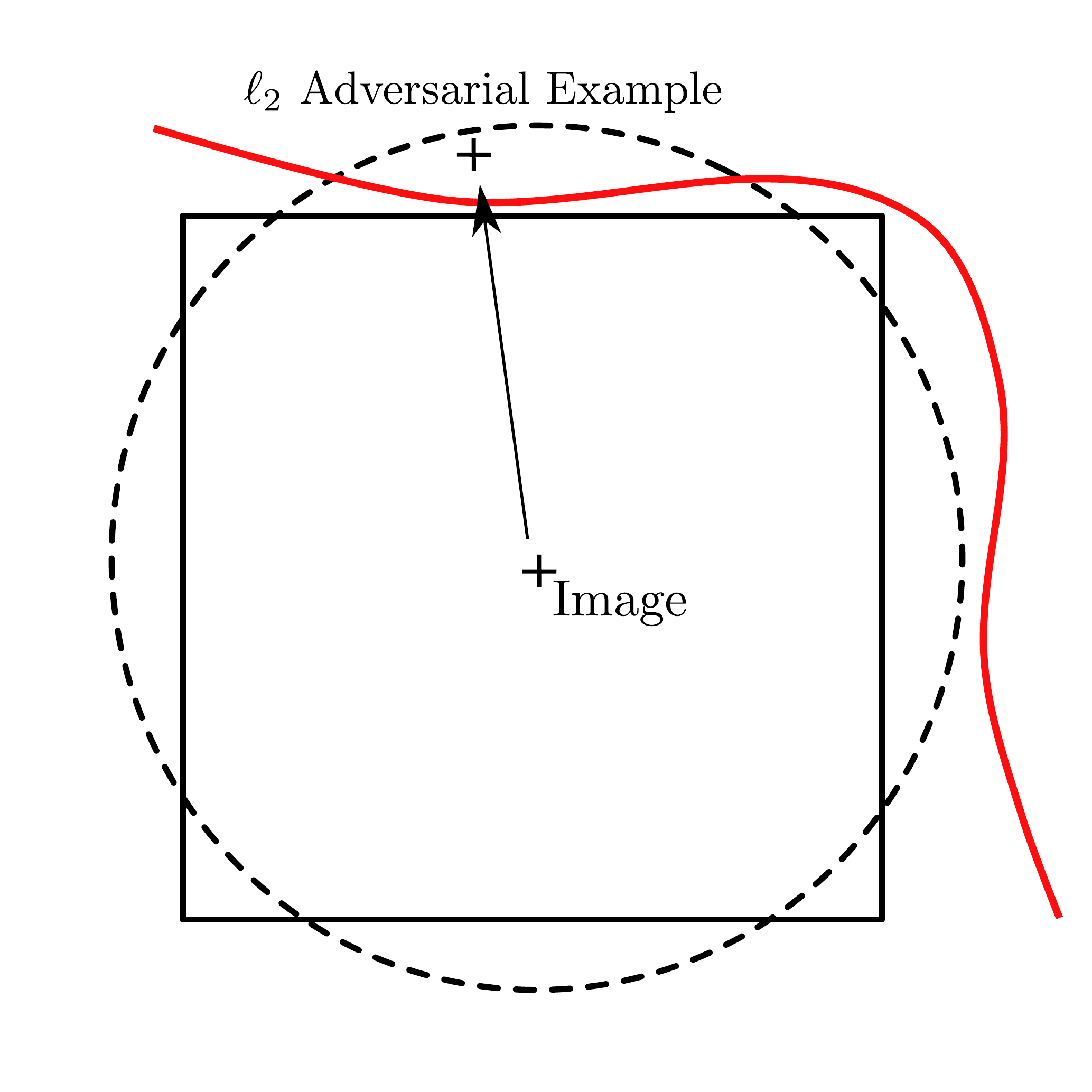}\\(b)
  \end{minipage}
  \begin{minipage}{.33\linewidth}
      \centering
      \includegraphics[scale=0.15]{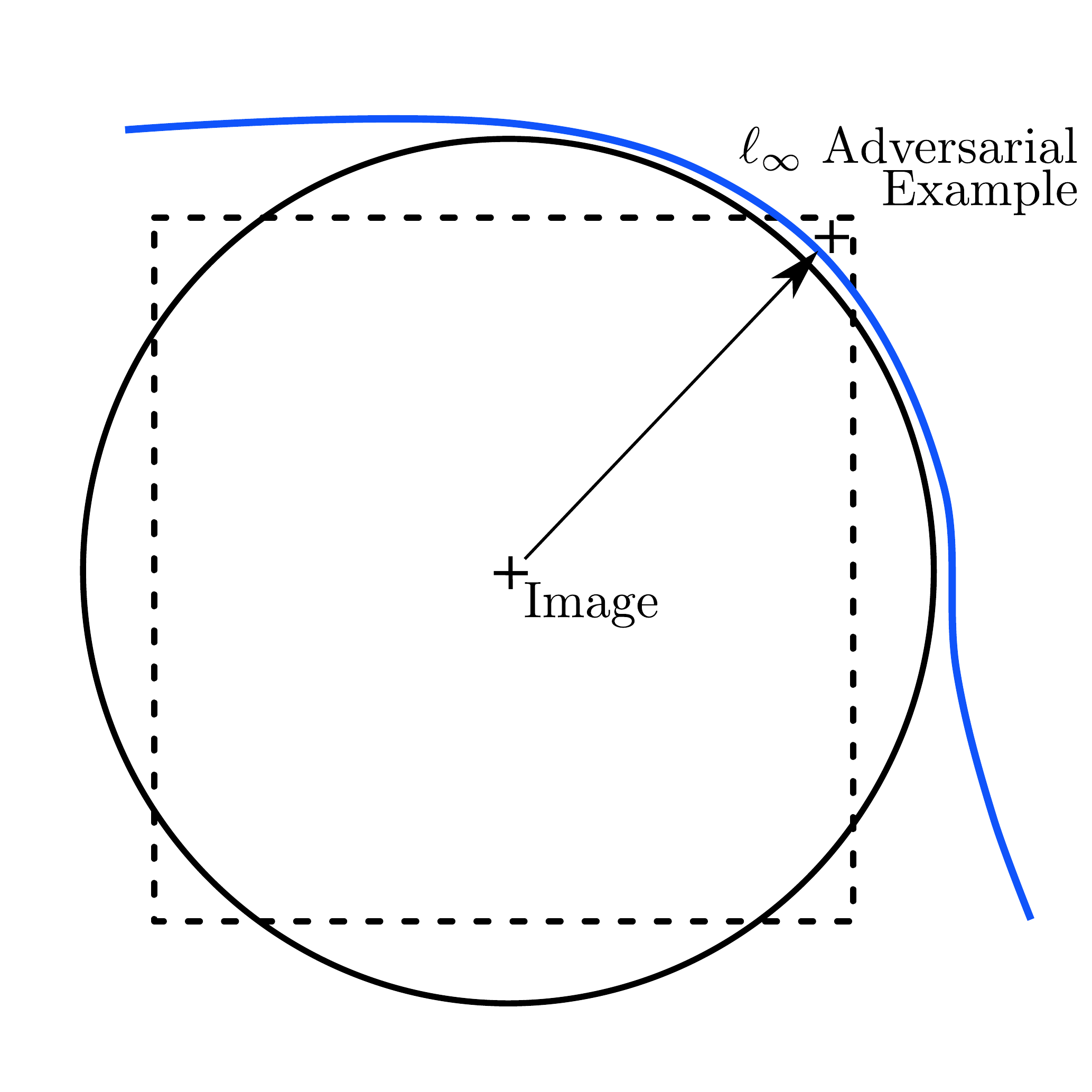}\\(c)
  \end{minipage}
    \caption{ Left: 2D representation of the \linf and \ltwo balls of respective radius $\epsilon$ and $\epsilon'$. 
    Middle: a classifier trained with \linf adversarial perturbations  (materialized by the red line) remains vulnerable to \ltwo attacks. 
    Right: a classifier trained with \ltwo adversarial perturbations (materialized by the blue line) remains vulnerable to \linf adversarial examples.}
  \label{figure:balls}
\end{figure*}%

\section{No Free Lunch for adversarial defenses}
\label{sec:no_free_lunch}

\subsection{Theoretical analysis}

Let us consider a classifier $f_{\epsilon_\infty}$ equipped with an {\em ideal defense mechanism} against adversarial examples bounded with an $\ell_\infty$ norm of value $\epsilon_\infty$. It guarantees that for any input-output pair $ (x,y) \sim \mathcal D$ and for any perturbation $\tau$ such that $\norm{\tau}_\infty \leq \epsilon_\infty$, $f_{\epsilon_\infty}$ is not misled by the perturbation (i.e. $f_{\epsilon_\infty}(x + \tau) = f_{\epsilon_\infty}(x)$).
We now focus our study on the performance of this classifier against adversarial examples bounded with an \ltwo norm of value $\epsilon_2$. 

Using Figure~\ref{figure:balls}(a), we observe that any \ltwo adversarial example that is also in the \linf ball, is guaranteed to be protected by the \linf defense mechanism of $f_{\epsilon_\infty}$, but not if it is outside the \linf ball. 
To characterize the probability that an  $\ell_2$ perturbation is guaranteed to be protected by an $\ell_\infty$ defense mechanism in the general case (i.e. any dimension $d$), we measure the ratio between the volume of intersection of the $\ell_\infty$ ball of radius $\epsilon_\infty$ and the $\ell_2$ ball of radius $\epsilon_2$. As Theorem~\ref{theorem:nullvolume} shows, this ratio depends on the dimensionality $d$ of the input vector $x$, and  rapidly converges to zero when $d$ increases. 
Therefore a defense mechanism that protects against all \linf bounded adversarial examples, is unlikely to be efficient against \ltwo attacks.

\begin{theorem}[Probability of the intersection goes to $0$]
\label{theorem:nullvolume}
Let $B_{2,d}$, and $B_{\infty,d}$ be two $d$ dimensional balls, respectively for \ltwo norm and \linf norm. If for all $d$, one constrains $B_{2,d}$, and $B_{\infty,d}$ to have the same volume, then $$\frac{\Vol(B_{2,d}\bigcap B_{\infty,d})}{\Vol(B_{\infty,d})} \rightarrow 0 \text{ when } d\rightarrow \infty. $$
\end{theorem} 
\begin{proof} 
Without loss of generality, let us fix the radius of the $\linf$ ball to $1$ (denoted $B_{\infty,d}(1)$). One can show that for all $d$, $\Vol\left( B_{2,d}\left(r_2(d)\right)\right) = \Vol\left(B_{\infty,d}\left(1\right)\right)$. Where $r_2(d)=\frac{2}{\sqrt{\pi}}\Gamma(\frac{d}{2}+1)^{1/d}$, $\Gamma$ is the gamma function, and $B_{2,d}\left(r_2(d)\right)$ is the \ltwo ball of radius $r_2(d)$. Then, thanks to Stirling's formula,  $r_2(d)\sim \sqrt{\frac{2}{\pi e}} d^{1/2}$. Finally, if we denote $\mathcal{U}_S$, the uniform distribution on set $S$, by using  Hoeffding inequality between Eq. (9) and (10), we get:

\begin{align}
&\frac{\Vol(B_{2,d}(r_2(d))\bigcap B_{\infty,d}(1))}{\Vol(B_{\infty,d}(1))} \\
=&\prob_{x\sim \mathcal{U}_{B_{\infty,d}(1)}}\left[x\in B_{2,d}(r_2(d))\right] \\
=&\prob_{x\sim \mathcal{U}_{B_{\infty,d}(1)}}\left[\textstyle \sum_{i=1}^d |x_i|^2\leq r_2^2(d)\right]\\
\leq &\exp{- d^{-1} \left( r_2^2(d)-d\mathbb{E}|x_1|^2\right)^2} \\
\leq &\exp{-\left( \frac{2}{\pi e}-\frac23\right)^2d+o(d)}. 
\end{align}
\noindent
Then the ratio between the volume of intersection of the ball and the volume of the ball converges towards 0.
\end{proof}
Theorem~\ref{theorem:nullvolume} states that, when $d$ is large enough, \ltwo bounded perturbations have a null probability of being also in the \linf ball of the same volume. As a consequence, for any value of $d$ that is large enough, a defense mechanism that offers full protection against $\linf$ adversarial examples is not guaranteed to offer any protection against $\ltwo$ attacks, and vice-versa\footnote{Th. \ref{theorem:nullvolume} can easily be extended to any two balls with different norms. For clarity, we restrict to the case of \linf and \ltwo norms.}.

Remark that this result defeats the 2-dimensional intuition: if we consider a 2 dimensional problem setting, the \linf and the \ltwo balls have an important overlap (as illustrated in Figure~\ref{figure:balls}(a)) and the probability of sampling in the intersection of the two balls is bounded by approximately 98\%. However, as we increase the dimensionality $d$, this probability quickly becomes negligible, even for very simple image datasets such as MNIST. An instantiation of  the bound for classical image datasets is presented in Table~\ref{table:datadim}. The probability of sampling in the intersection of the \linf and \ltwo balls is close to zero for any realistic image setting. In large dimensions, the volume of the corner of the \linf ball is much bigger than it appears in Figure~\ref{figure:balls}(a).

\begin{table}[h]
\centering
{\footnotesize
\begin{tabular}{c|r|r}
\toprule
\textbf{Dataset} & \multicolumn{1}{c|}{\boldmath{}$d$\unboldmath{}}  & \multicolumn{1}{c}{\textbf{Inter.} (in $\log_{10}$)}\\
\midrule
-- & 2 & -0.009 \\
MNIST & 784  & -144\\
CIFAR& 3072 & -578\\
ImageNet& 150528 &-28946\\
\bottomrule
\end{tabular}}
\caption{Bounds of Theorem~\ref{theorem:nullvolume} on the volume of intersection of  $\ell_2$ and $\ell_\infty$ balls at equal volume for typical image classification datasets. When $d=2$, the bound is $ 10^{-0.009}\sim 0.98$.}
\label{table:datadim}
\end{table}

\subsection{No Free Lunch in practice}

Our theoretical analysis shows that if adversarial examples were uniformly distributed in a high dimensional space, then any mechanism that perfectly defends against \linf adversarial examples has a null probability of protecting against \ltwo-bounded adversarial attacks and vice-versa. Although existing defense mechanisms do not necessarily assume such a distribution of adversarial examples, we demonstrate that whatever distribution they use, it offers no favorable bias w.r.t the result in Theorem~\ref{theorem:nullvolume}. 
As we discuss in Sec.~\ref{sec:preliminaries}, there are two distinctive attack settings: loss maximization (PGD) and perturbation minimization (C\&W). We analyse the first setting in details and conduct a second series of experiments to demonstrate that the results are similar if we consider the second setting.

\paragraph{Adversarial training vs. loss maximization attacks}
To demonstrate that \linf adversarial training is not robust against PGD-\ltwo attacks, we measure the number of \ltwo adversarial examples generated with PGD-\ltwo, lying outside the \linf ball. (Note that we consider {\em all} examples, not just the ones that successfully fool the classifier). To do so, we use the same experimental setting as in Section~\ref{sec:building_defense_mechanisms} with $\epsilon_\infty$ and $\epsilon_2$ such that the volumes of the two balls are equal. Additionally, we also measure the average \linf and \ltwo norms of these adversarial examples, to understand more precisely the impact of  adversarial training, and we report the accuracy, which reflects the number of adversarial examples that successfully fooled the classifier (cf. Table~\ref{tab:mean_norm_pgd_attack_ben}~(top)). The same experiment is conducted for \ltwo adversarial training against PGD-\linf and the results are presented in Table~\ref{tab:mean_norm_pgd_attack_ben}~(bottom). All experiments in this section are conducted on CIFAR-10, and the experimental setting is fully detailed in Section~\ref{sec:experimental_settings}.

\begin{table}[h!]
  \centering
  {\footnotesize
    \begin{tabular}{lrrr}
    \toprule
    \textbf{PGD-}\ltwo \textbf{vs.} \boldmath{}$\rightarrow$\unboldmath{} & \textbf{Unprotected} & \textbf{AT-}\linf \\\midrule
     Examples inside \ltwo ball & 100\% & 100\% \\
     Average \ltwo norm & 0.83 & 0.83 \\\midrule
     Examples inside \linf ball & 0\% &  0\% \\
     Average \linf norm & 0.075 & 0.2\\\midrule
     Accuracy under attack & 0.00 & 0.37\\
    \bottomrule \\
    \toprule
    \textbf{PGD-}\linf \textbf{vs.} \boldmath{}$\rightarrow$\unboldmath{} & \textbf{Unprotected} & \textbf{AT-}\ltwo \\\midrule
     Examples inside \ltwo ball & 100\% & 100\% \\
     Average \ltwo norm & 1.4 & 1.64\\\midrule
     Examples inside \linf ball & 0\% &  0\%  \\
     Average \linf norm &0.031 &0.031\\\midrule
     Accuracy under attack & 0.00 & 0.37\\
    \bottomrule
    \end{tabular}}
    \caption{(Top) number of PGD-\ltwo adversarial examples inside the \linf and inside the \ltwo ball, without and with \linf adversarial training. (Bottom) number of PGD-\linf adversarial examples inside the \linf and inside the \ltwo ball, without and with \ltwo adversarial training. On CIFAR-10 ($d=3072$).}
  \label{tab:mean_norm_pgd_attack_ben}%
\end{table}%

The results are unambiguous: {\em none} of the adversarial examples generated with PGD-\ltwo are inside the \linf ball (and thus in the intersection of the two balls). As a consequence, we cannot expect adversarial training \linf to offer any guaranteed protection against \ltwo adversarial examples. We illustrate this phenomenon using Figure~\ref{figure:balls}~(b): notice that the \ltwo adversarial example represented in this figure cannot be protected using \linf adversarial training which is only designed to push the decision boundary (red line) outside of the \linf ball (square), but not outside of the \ltwo ball (circle). Our results demonstrate that {\em all} PGD-\ltwo examples are already in this upper area (outside the intersection), before \linf adversarial training. Therefore \linf adversarial training is unnecessary. 

The second experiment naturally demonstrates a similar behaviour. We first observe that adversarial examples generated with PGD-\linf lying outside the \ltwo ball cannot be eliminated using \ltwo adversarial training (as illustrated in Figure~\ref{figure:balls}~(c)). However,  Table~\ref{tab:mean_norm_pgd_attack_ben} shows that all examples are already outside the \ltwo ball, clustered around the corner of the \linf ball (average distance is 1.64 compared to $0.031 \times \sqrt{3072} = 1.71$ for the corner). Therefore, any defense method (including \ltwo adversarial training) that would eliminate only adversarial examples inside the \ltwo balls, cannot be efficient against \linf adversarial examples. 

The comparison of accuracy under PGD-\ltwo attack of a classifier defended by either \linf or \ltwo adversarial training corroborate our analysis. In fact, when defended with AT-\linf the accuracy of the classifier under attack is $0.37$, while the AT-\ltwo defends the classifier up to $0.52$ i.e. $40\%$ better. Similarly, a classifier defended with AT-\linf with an accuracy under PGD-\linf attack of $0.43$ performs $16\%$ better than the one defended with AT-\ltwo which obtains $0.37$ accuracy under attack. These results keep confirming our claim: \ltwo-based defenses are inadequate to defend against \linf attacks, and vice-versa.

\begin{table}[htbp]
  \centering
  {\footnotesize
    \begin{tabular}{l|r|r}
    \toprule
      & \textbf{Unprotected} & \textbf{AT}-$\ell_\infty$ \\
    \midrule
    Examples inside intersection & 70\% & 29\% \\
    Examples outside intersection & 30\% & 71\% \\\hline
    Accuracy under attack & 0.00  & 0.00 \\
    \bottomrule
    \end{tabular}}
  \caption{This table shows the amount of adversarial examples inside the $\ell_\infty$ ball and inside the $\ell_2$ ball but outside the $\ell_\infty$ ball. We can observe a clear shift between a baseline model (no defense) and a model trained with Adversarial Training PGD $\ell_\infty$ attacked with C\&W attack \cite{carlini2017towards}.}
  \label{tab:balls_of_carlini_attack}
\end{table}%

\paragraph{Adversarial training vs. perturbation minimization attacks.}
We now study the performances of an \ltwo perturbation maximization attack (C\&W) with and without AT-\linf. It allows us to understand in which area C\&W discovers adversarial examples and the impact of AT-\linf. The results are reported in Table~\ref{tab:balls_of_carlini_attack}. First, when the classifier is undefended, we observe that $70\%$ of adversarial examples lie inside the intersection of the two balls. This phenomenon is due to the fact that C\&W minimizes the \ltwo norm of the perturbation. Therefore without AT, the attack is able to discover adversarial examples that are very close to the original image, where the \linf and the \ltwo balls overlap.  When the model is trained with AT-\linf, we observe a clear shift: $71\%$ of the examples are now outside the \linf, but still inside the $\ell_2$ ball, as illustrated in Figure~\ref{figure:balls}~(b). This means that C\&W attack still minimizes the \ltwo norm of the perturbation while updating its search space to ignore the examples in the \linf ball. Since C\&W was always able to discover adversarial examples in this area, AT-\linf offers no extra benefit in terms of robustness (0\% Accuracy). Together, these results and Theorem~\ref{theorem:nullvolume} confirm that $\ell_\infty$-based defenses are vulnerable to \ltwo-based perturbation minimization attacks.

\begin{table*}[htp]
  \centering
    {\scriptsize
\begin{tabular}{l||r||r|r|r|r||r|r|r|r||r|r|r|r}
\toprule
  & \multicolumn{1}{c||}{\textbf{Baseline}} & \multicolumn{2}{c|}{\textbf{AT}} & \multicolumn{2}{c||}{\textbf{MAT}} & \multicolumn{2}{c|}{\textbf{NI}} & \multicolumn{2}{c||}{\textbf{MNI}} & \multicolumn{2}{c|}{\textbf{RAT-}$\ell_\infty$} & \multicolumn{2}{c}{\textbf{RAT-}$\ell_2$} \\
\cmidrule{2-14}  & \multicolumn{1}{c||}{--} & \multicolumn{1}{c|}{$\ell_\infty$} & \multicolumn{1}{c|}{$\ell_2$} & \multicolumn{1}{c|}{Max} & \multicolumn{1}{c||}{Rand} & \multicolumn{1}{c|}{$\mathcal{N}$} & \multicolumn{1}{c|}{$\mathcal{U}$} & \multicolumn{1}{c|}{Mix} & \multicolumn{1}{c||}{Conv} & \multicolumn{1}{c|}{$\mathcal{N}$} & \multicolumn{1}{c|}{$\mathcal{U}$} & \multicolumn{1}{c|}{$\mathcal{N}$} & \multicolumn{1}{c}{$\mathcal{U}$} \\
\midrule
\multicolumn{1}{c||}{Natural examples} & 0.94 & 0.85 & 0.85 & 0.80 & 0.80 & 0.79 & 0.87 & 0.84 & 0.79 & \cellcolor[rgb]{ .878,  .949,  .902}0.74 & \cellcolor[rgb]{ .878,  .949,  .902}0.80 & 0.79 & 0.87 \\
PGD-$\ell_\infty$ 20  & \cellcolor[rgb]{ .996,  .671,  .667}0.00 & 0.43 & 0.37 & 0.37 & 0.40 & \cellcolor[rgb]{ .996,  .671,  .667}0.23 & \cellcolor[rgb]{ .996,  .671,  .667}0.22 & \cellcolor[rgb]{ .996,  .671,  .667}0.19 & \cellcolor[rgb]{ .996,  .671,  .667}0.20 & \cellcolor[rgb]{ .878,  .949,  .902}0.35 & \cellcolor[rgb]{ .878,  .949,  .902}0.40 & \cellcolor[rgb]{ .996,  .671,  .667}0.23 & \cellcolor[rgb]{ .996,  .671,  .667}0.22 \\
PGD-$\ell_2$ 20 & \cellcolor[rgb]{ .996,  .671,  .667}0.00 & 0.37 & 0.52 & 0.50 & 0.55 & 0.34 & 0.36 & 0.33 & 0.32 & \cellcolor[rgb]{ .878,  .949,  .902}0.43 & \cellcolor[rgb]{ .878,  .949,  .902}0.39 & 0.34 & 0.37 \\
C\&W-$\ell_2$  60 & \cellcolor[rgb]{ .996,  .671,  .667}0.00 & \cellcolor[rgb]{ .996,  .671,  .667}0.00 & \cellcolor[rgb]{ .996,  .671,  .667}0.00 & \cellcolor[rgb]{ .996,  .671,  .667}0.00 & \cellcolor[rgb]{ .996,  .671,  .667}0.00 & 0.33 & 0.53 & 0.41 & 0.32 & \cellcolor[rgb]{ .878,  .949,  .902}0.30 & \cellcolor[rgb]{ .878,  .949,  .902}0.41 & 0.33 & 0.34 \\
\midrule
\textbf{Min Accuracy} & 0.00 & 0.00 & 0.00 & 0.00 & 0.00 & 0.23 & 0.22 & 0.19 & 0.20 & \textbf{0.30} & \textbf{0.39} & 0.23 & 0.22 \\
\bottomrule
\end{tabular}%
}
  \caption{This table shows a comprehensive list of results consisting of the accuracy of several defense mechanisms against $\ell_2$ and $\ell_\infty$ attacks. This table main objective is to compare the overall performance of ‘single‘ norm defense mechanisms (AT and NI presented in the Sec.~\ref{subsec:defense_mechanisms}) against mixed norms defense mechanisms (MNI, MAT \& RAT mixed defenses presented in Sec.~\ref{sec:building_defense_mechanisms}). The red values present all accuracy \emph{below} 30\% which shows that all defense mechanisms have ‘weaknesses‘ with the exception of RAT.}
  \label{tab:results}
\end{table*}

\section{Building Defenses against Multiple Adversarial Attacks}
\label{sec:building_defense_mechanisms}

So far, we have shown that adversarial defenses are able to protect only against the norm they have been trained on. In order to solve this problem, we propose several strategies to build defenses against multiple adversarial attacks. These strategies are based on the idea that both types of defense must be used simultaneously in order for the classifier to be protected against multiple attacks. In this section we evaluate several of these defense strategies, and compare them against state-of-the-art attacks using a solid experimental setting (the detailed description of the  experimental setting is described in Section~\ref{sec:experimental_settings}).

\subsection{MAT -- Mixed Adversarial Training}\label{subsec:mixed_adversarial_training}
Earlier results have shown that AT-$\ell_p$ improves the robustness against corresponding $\ell_p$-bounded adversarial examples, and the experiments we present in this section corroborate this observation (See Table~\ref{tab:results}, column: AT). Building on this observation, it is natural to examine the efficiency of Mixed Adversarial Training (MAT) against mixed \linf and \ltwo attacks. MAT is a variation of AT that uses both \linf-bounded adversarial examples and \ltwo-bounded adversarial examples as training examples. 

As discussed by~\citet{tramer2019adversarial}, there are several possible strategies to mix the adversarial training examples. The first strategy (MAT-Rand) consists in randomly selecting one adversarial example among the two most damaging \linf and \ltwo, and to use it as a training example, as described in Equation~\ref{eq:mat-rand}:

\noindent
{\em MAT-Rand:}
\begin{equation}
    \min_{\theta}\expect_{(x, y) \sim \mathcal{D}} \left[\expect_{p\sim\mathcal{U}({\{2, \infty\})}} \left[ \max_{\norm{\tau}_p \leq \epsilon} \mathcal{L} \left( f_{\theta}(x+\tau), y \right) \right] \right].
    \label{eq:mat-rand}
\end{equation}

An alternative strategy is to systematically train the model with the most damaging adversarial example (\linf or \ltwo). As described in Equation~\ref{eq:mat-max}: 

\noindent
{\em MAT-Max:}
\begin{equation}
    \min_{\theta}\expect_{(x, y) \sim \mathcal{D}} \left[ \max_{p \in \{2, \infty\}} \max_{\norm{\tau}_p \leq \epsilon} \mathcal{L} \left( f_{\theta}(x+\tau), y \right) \right].
    \label{eq:mat-max}
\end{equation}
\noindent

The accuracy of MAT-Rand and MAT-Max are reported in Table~\ref{tab:results} (Column: MAT). As expected, we observe that MAT-Rand and MAT-Max offer better robustness both against PGD-\ltwo and PGD-\linf adversarial examples than the original AT does. More  generally, we can see that AT is a good strategy against loss maximization attacks, and thus it is not surprising that MAT is a good strategy against mixed loss maximization attacks. However, AT is very weak against perturbations minimization attacks such as C\&W, and MAT is no better against such attacks. This weakness makes MAT of little practical use.

\subsection{MNI -- Multiple Noise Injection}

Another important technique to defend against adversarial examples is to use Noise Injection (NI). \citet{pinot2019theoretical} demonstrated that injecting noise in the network can give provable defense against adversarial examples. Furthermore, we found that NI offers better protection than AT against perturbation minimization attacks such as C\&W, thus, they are good candidates to obtain models robust to multiple attacks. In this work, besides the generalized Gaussian noises, already investigated in previous works, we evaluate the efficiency of uniform distributions which are generalized Gaussian of order $\infty$. As shown in Table~\ref{tab:results} (Columns: NI), noise injection from this distribution gives better results than Gaussian noise injection against all the attacks except PGD-\linf. 

To obtain the best out of both noises, we propose to combine them (MNI) either by convolution (Conv) or by mixture (Mix). Hence, the final noise vector comes from one of the following probability density functions:\\
{\em MNI-Conv}:
\begin{equation}
 \frac{1}{\sqrt{2 \pi \sigma_1^2}} \exp\left\{ \frac{-x^2}{2 \sigma_1^2}\right\}  * \frac{\mathds{1}\{ |x| \leq \sigma_2\}}{2 \sigma_2}
\end{equation}
{\em MNI-Mix}:
\begin{equation}
\frac{1}{\sqrt{8 \pi \sigma_1^2}} \exp\left\{ \frac{-x^2}{2 \sigma_1^2}\right\}   + \frac{\mathds{1}\{ |x| \leq \sigma_2\}}{4 \sigma_2}.
\end{equation}
Following the literature~\cite{pinot2019theoretical}, we choose $\sigma_1=0.25$. Accordingly, we take $\sigma_2=0.2$. 
The results are presented in Table~\ref{tab:results}~(Column: MNI).
We found that MNI offers comparable results against the experimental setting in \cite{pinot2019theoretical}, but does not improve over NI with a uniform distribution. 

\subsection{RAT -- Randomized Adversarial Training }\label{subsec:randomized_adversarial_training}

We now examine the performance of Randomized Adversarial Training (RAT) which mixes Adversarial Training with Noise Injection. We consider the two symmetric settings: RAT-\linf and a noise from a normal distribution, as well as RAT-\ltwo and a noise from a uniform distribution. The corresponding loss function is defined as follows: \begin{equation}
    \min_{\theta}\expect_{(x, y) \sim \mathcal{D}} \left[ \max_{\norm{\tau} \leq \epsilon} \mathcal{L} \left( \tilde{f}_{\theta}(x+\tau), y)  \right) \right].
\end{equation}
\noindent where $\tilde{f}_\theta$ is a randomized neural network with noise injection as described in Section~\ref{subsec:randomized_training}.

The results of RAT are reported in Table~\ref{tab:results}~(Columns: RAT-\linf and RAT-\ltwo).
We can observe that the first setting offers the best extra robustness, which is consistent with previous experiments, since AT is generally more effective against \linf attacks whereas NI is more effective against \ltwo-attacks. Overall, RAT-\linf and a noise from uniform distribution offer the best minimal robustness with at least $0.39$ accuracy, 16 points above the second best (NI with noise from a normal distribution, with 0.22).

\subsection{Experimental setting}
\label{sec:experimental_settings}

To compare the robustness provided by the different defense mechanisms, we use strong adversarial attacks and a conservative setting: the attacker has a total knowledge of the parameters of the model (white-box setting) and we only consider untargeted attacks  (a misclassification from one target to any other will be considered as adversarial). To evaluate defenses based on noise injection, we use {\em Expectation Over Transformation} (EOT), the rigorous experimental protocol  proposed by \citet{athalye2017synthesizing} and later used by \citet{athalye2018obfuscated,carlini2019evaluating} to identify flawed defense mechanisms. 

To attack the models, we use state-of-the-art algorithms PGD and C\&W (see Section~\ref{sec:preliminaries}). We run PGD with 20 iterations to generate adversarial examples and with 10 iterations when it is used for adversarial training. We run C\&W with 60 iterations to generate adversarial examples. For bounded attacks, the maximum \linf bound is fixed to $0.031$ and the maximum \ltwo bound is fixed to $0.83$. As discussed in Section~\ref{sec:preliminaries}, we chose these values so that the \linf and the \ltwo balls have similar volumes. Note that $0.83$ is slightly above the values typically used in previous publications in the area, meaning the attacks are stronger, and thus  more difficult to defend against.

All experiments are conducted on CIFAR-10 with the Wide-Resnet 28-10 architecture. We use the training procedure and the hyper-parameters described in the original paper by~\citet{zagoruyko2016wide}. Training time varies from 1 day (AT) to 2 days (MAT) on 4 GPUs-V100 servers.

\section{Related Work}

Adversarial attacks have been an active topic in the machine learning community since their discovery~\cite{globerson2006nightmare, biggio2013evasion,Szegedy2013IntriguingPO}. Many attacks have been developed. Most of them solve a loss maximization problem with either $\ell_\infty$~\cite{goodfellow2014explaining,kurakin2016adversarial,madry2018towards}, $\ell_2$~\cite{carlini2017towards,kurakin2016adversarial,madry2018towards}, $\ell_1$~\cite{tramer2019adversarial} or $\ell_0$~\cite{papernot2016limitations} surrogate norms. 

Defending against adversarial examples is a challenging problem since the number of layers makes it difficult to understand the geometry of the decision boundary. Despite empirically proven efficient, Adversarial training~\cite{goodfellow2014explaining} gives no formal defense guarantees. Besides this line of work, randomization and smoothing~\cite{xie2018mitigating,lecuyer2018certified,pinot2019theoretical,KolterRandomizedSmoothing} have gained popularity since they provide guarantees, but so far, the efficiency of these methods remains limited against $\ell_\infty$-based attacks.   

An open question so far is to build an efficient defense against multiple norms. Concurrently to our work, \citet{tramer2019adversarial} proposed to tackle this issue by mixing randomized training with attacks for different norms to defend against multiple perturbations. Then, \citet{salman2019provably} proposed to mix adversarial training with randomized smoothing to have better certificates against adversarial attacks. These methods are closely related respectively to MAT and RAT. Aside from these similarities, we propose a new geometric point of view for robustness against multiple perturbations, that is backed up theoretically and experimentally. We also conduct a rigorous and full comparison of RAT and MAT as defenses against adversarial attacks. Finally, we propose MNI, that adds mixture of noise to our network and gets promising results. To the best of our knowledge, this is the first work that covers mixtures and convolution of noises with different natures.

\section{Conclusion}
In this paper, we tackle the problem of protecting neural networks against multiple attacks crafted from different norms. First, we demonstrate that existing defense mechanisms can only protect against one type of attacks. Then we consider a variety of strategies to mix defense mechanisms and to build models that are robust against multiple adversarial attacks. We show that {\em Randomized Adversarial Training} offers the best global performance.

\clearpage
\newpage

\bibliographystyle{named}
\bibliography{bibliography}

\end{document}